\documentclass[twoside]{article}
\usepackage[accepted]{aistats2e}
\usepackage{amsmath,amssymb}
\usepackage{graphicx}
\usepackage{pgfplots}
\pgfplotsset{width=10cm,compat=1.9}
\usepackage[round]{natbib}
\usepackage{hyperref}
\usepackage{enumitem}
\usepackage{subcaption}
\usetikzlibrary{automata}
\usetikzlibrary{calc}
\usepackage{subcaption}
\usepgfplotslibrary{fillbetween}

\usepackage{amsthm}
\newtheorem{theorem}{Theorem}[section]
\newtheorem{corollary}{Corollary}[theorem]
\newtheorem{lemma}[theorem]{Lemma}

\numberwithin{intassumption}{assumption}
\newtheorem{definition}{Definition}[section]

\begin{document}

\twocolumn[

\aistatstitle{Preventing Manifold Intrusion with Locality: Local Mixup}

\aistatsauthor{ Raphaël Baena\And Lucas Drumetz \And Vincent Gripon  }

\aistatsaddress{IMT Atlantique, Lab-STICC, UMR CNRS 6285, F-29238, France} ]

\begin{abstract}

Mixup is a data-dependent regularization technique that consists in linearly interpolating input samples and associated outputs. It has been shown to improve accuracy when used to train on standard machine learning datasets. However, authors have pointed out that Mixup can produce out-of-distribution virtual samples and even contradictions in the augmented training set, potentially resulting in adversarial effects. In this paper, we introduce \emph{Local Mixup} in which distant input samples are weighted down when computing the loss. In constrained settings we demonstrate that \emph{Local Mixup} can create a trade-off between bias and variance, with the extreme cases reducing to vanilla training and classical Mixup. Using standardized computer vision benchmarks , we also show that \emph{Local Mixup} can improve test accuracy.

\end{abstract}

\section{Introduction}

Deep Learning has become the golden standard for many tasks in the fields of machine learning and signal processing. Using a large number of tunable parameters, Deep Neural Networks (DNNs) are able to identify subtle dependencies in large training datasets to be later leveraged to perform accurate predictions on previously unseen data. Without constraints or enough samples, many models can fit the training data (high variance) and it is difficult to find the ones that would generalize correctly (low bias).

Regularization techniques have been deployed with the aim of improving generalization~\citep{goodfellow2016deep}. In~\citep{guo2019mixup}, the authors categorize theses techniques into data-independent or data-dependent ones. For example some data-independent regularization techniques constrain the model by penalizing the norm of the parameters, for instance through weight decay~\citep{weightdecay}. A popular data-dependent regularization technique consists of artificially increasing the size of the training set, which is referred to as \emph{data augmentation}~\citep{dataaugmentation}. In the field of computer vision, for example, it is very common to generate new samples using basic class-invariant transformations~\citep{horizontalflip,horizontalflip3}. 

In~\citep{zhang2017mixup}, the authors introduce \emph{Mixup}, a data augmentation technique in which artificial training samples $(\mathbf{\tilde x}, \mathbf{\tilde y})$, called \emph{virtual} samples, are generated through linear interpolations between two training samples $(\mathbf{x}_i,\mathbf{y}_i)$ and $(\mathbf{x}_j, \mathbf{y}_j)$. The associated output is computed as the corresponding linear interpolation on the respective outputs. \emph{Mixup} improves generalization error of state-of-the-art models on ImageNet, CIFAR, speech, and tabular datasets~\citep{zhang2017mixup}. This method is also used in the context of few shot learning~\citep{fewshot1,fewshot2}.

By using linear interpolation, \emph{virtual} samples can in some cases contradict each other, or even generate out-of-distribution inputs. This phenomenon has been recently described in~\citep{guo2019mixup} where the authors use the term \emph{manifold intrusion}. As such, it is not clear if \emph{Mixup} is always desirable. More generally, the question arises of whether \emph{Mixup} could be constrained to reduce the risk of generating such spurious interpolations. In this paper we introduce \emph{Local Mixup}, where \emph{virtual} samples are weighted in the training loss. The weight of each possible \emph{virtual} sample depends on the distance between the endpoints of the corresponding segment $(\mathbf{x}_i,\mathbf{x}_j)$. In particular, this method can be implemented to forbid interpolations between samples that are too distant from each other in the input domain, reducing the risk of generating spurious virtual samples.

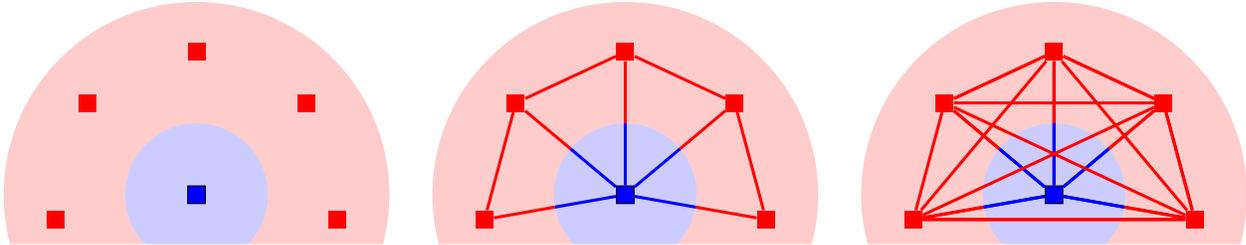
\begin{figure*}
    \centering
    
\begin{tikzpicture}[nodes=state, scale=0.95]

\tikzstyle{every node} = [minimum width=3pt];
\begin{scope}[rotate=90, yshift=12cm]
\def \number {5}
\def \radius {2cm}
\def \degree {250/\number}

\draw[draw=none,fill=blue,opacity=0.2,] (0,0) circle (1cm);
\draw[draw=none,fill=red,opacity=0.2,even odd rule] (0,0) circle (1cm) (0,0) circle (2.7cm);
\draw[fill=white,draw = white] (-3,-3) rectangle (-0.7, 3);

\foreach \s in {1,...,\number}
{
    \node[fill =red] at ({\degree * (\s-3 )}:\radius) (\s) {};
}

\node[draw,fill = blue] at (0,0) (B) {};

\end{scope}
\begin{scope}[rotate=90, yshift=6cm]
\def \number {5}
\def \radius {2cm}
\def \degree {250/\number}

\draw[draw=none,fill=blue,opacity=0.2,] (0,0) circle (1cm);
\draw[draw=none,fill=red,opacity=0.2,even odd rule] (0,0) circle (1cm) (0,0) circle (2.7cm);
\draw[fill=white,draw = white] (-3,-3) rectangle (-0.7, 3);

\foreach \s in {1,...,\number}
{
    \node[fill =red] at ({\degree * (\s-3 )}:\radius) (\s) {};
}

\node[draw,fill = blue] at (0,0) (B) {};

\foreach \s [evaluate=\s as \result using {int(Mod( (\s), \number)+1)}]in {1,...,4}
{
     \draw[line width=0.4mm] (\s) -- (\result) [red];
}
\foreach \s in {1,...,\number}
{
    \draw[red,line width=0.4mm]  (\s) -- ($(\s)!0.5!(B)$) ;
    \draw[blue,line width=0.4mm]  (B) -- ($(B)!0.5!(\s)$) ;
}

\end{scope}
\begin{scope}[rotate=90]
\def \number {5}
\def \radius {2cm}
\def \degree {250/\number}

\draw[draw=none,fill=blue,opacity=0.2,] (0,0) circle (1cm);
\draw[draw=none,fill=red,opacity=0.2,even odd rule] (0,0) circle (1cm) (0,0) circle (2.7cm);
\draw[fill=white,draw = white] (-3,-3) rectangle (-0.7, 3);

\foreach \s in {1,...,\number}
{
    \node[fill =red] at ({\degree * (\s-3 )}:\radius) (\s) {};
}

\node[draw,fill = blue] at (0,0) (B) {};

\foreach \s in {1,...,\number}
{
    \draw[red,line width=0.4mm]  (\s) -- ($(\s)!0.5!(B)$) ;
    \draw[blue,line width=0.4mm]  (B) -- ($(B)!0.5!(\s)$) ;
}\foreach \s in {1,...,\number}
{
    \draw[red,line width=0.4mm]  (\s) -- ($(\s)!0.5!(B)$) ;
    \draw[blue,line width=0.4mm]  (B) -- ($(B)!0.5!(\s)$) ;
}
\foreach \s in {2,...,5}
{
    \draw[line width=0.4mm]  (\s) -- (1)[red];
}
\foreach \s in {3,...,5}
{
    \draw[line width=0.4mm]  (\s) -- (2)[red];
}
\foreach \s in {4,...,\number}
{
    \draw[line width=0.4mm]  (\s) -- (3)[red];
}

\foreach \s [evaluate=\s as \result using {int(Mod( (\s), \number)+1)}]in {1,...,5}
{
      \draw[line width=0.4mm] (\s) -- (\result) [red];
}
\end{scope}
\end{tikzpicture}
\vspace{-2.2cm}
    \caption{Illustration of the proposed \emph{Local Mixup} method. On the left, only vanilla samples are used, without data augmentation. Ground truth is depicted in filled regions. On the middle we depict \emph{Local Mixup} where we only interpolate samples which are close enough, leading to no contradiction with ground thuth. On the right we depict \emph{Mixup} in which we interpolate all samples, leading to contradictory virtual samples.}
    \label{fig:method}
    \vspace{-0.1cm}
\end{figure*}

Here are our main contributions:
\begin{itemize}[noitemsep,topsep=0pt]
    \item We introduce \emph{Local Mixup}, a mixup method depending on a single parameter whose extremes correspond to classical \emph{Mixup} and Vanilla.
    \item In dimension one, we prove that \emph{Local Mixup} allows to select a bias/variance trade-off.
    \item In higher dimensions, we show that \emph{Local Mixup} can help achieve more accurate models than classical \emph{Mixup} using standard vision datasets.
    \item Our work contributes more broadly to better understanding the impact of \emph{Mixup} during training.
\end{itemize}

\section{Related Work}

\textbf{Introducing notations:}
In Machine or Deep Learning, a training dataset $\mathcal{D}_{train}$ is used to learn the model's parameters, and a test one $\mathcal{D}_{test}$ is used to evaluate the performance of the model on previously unseen inputs~\citep{bishop2006pattern}. We also consider that both input and output data lie in metric spaces $(\mathcal{X},d_X)$ and $(\mathcal{Y},d_\mathcal{Y})$. Typically, $\mathcal{X}$ and $\mathcal{Y}$ are assumed to be Euclidean spaces with the usual metrics. We denote by $f:\mathcal{X} \to \mathcal{Y}$ the parametric model to be trained and by $\mathcal{F}$ the hypothesis set, i.e. the set containing all candidate parametrizations of the model  $f\in\mathcal{F}$.


To train our model, we use an \emph{error function} $\mathcal{L}$ that measures the discrepancy between the model outputs and expected ones. Training the model amounts to minimizing the training loss while generalization may be quantitavely evaluated by the test loss: 
\begin{align*}
    L_{vanilla} &= {\sum_{(\mathbf{x},\mathbf{y})\in \mathcal{D}}{\mathcal{L}(f(\mathbf{x}),\mathbf{y}}}). \\
\end{align*}

\textbf{Data augmentation and mixup:}
To improve generalization one can use regularization techniques~\citep{goodfellow2016deep}. Among them, data augmentation is a form of data-dependent regularization~\citep{guo2019mixup}.
It artificially generates new samples, resulting in increasing $\mathcal{D}_{train}$~\citep{dataaugmentation}, and can apply on the outputs $\mathbf{y}$~\citep{noisylabel} or on the inputs $\mathbf{x}$~\citep{colorful,cutoutorig,cutmixorig,autoaugment,horizontalflip,horizontalflip3}.




The use of data-dependent methods relying on some sort of \emph{mixing} has recently emerged~\citep{zhang2017mixup,manifold,cutmixorig,cutoutorig,augmix,puzzlemix,remix,automix,stackmix,batchmixup,mixmo}. They usually mix two or more inputs and the corresponding labels.
 
The pioneering mixing method is \emph{Mixup}~\citep{zhang2017mixup}, whose mixed samples $(\tilde x, \tilde y)$ are generated by linear interpolations between pairs of samples, i.e. $\mathbf{\tilde x}_{i,j,\lambda} = \lambda \mathbf{x}_i + (1-\lambda)\mathbf{x}_j$ and $\mathbf{\tilde y}_{i,j,\lambda} = \lambda \mathbf{y}_i + (1-\lambda)\mathbf{y}_j$ for some training samples $(\mathbf{x}_i,\mathbf{y}_i)$ and $(\mathbf{x}_j, \mathbf{y}_j)$ and some $\lambda \in [0,1]$ . The \emph{Mixup} training criterion is defined as:
\begin{definition}[Mixup Criterion] Let $\lambda~\sim~Beta[\alpha, \beta]$, $i,j$ discrete variables uniformly drawn with repetitions in $\{0, \dots, n-1 \}$. $f^*$ minimizes the Mixup criterion if: 
\begin{align*}
f^* = \underset{f \in \mathcal{F}}{\mathrm{argmin}} \  \frac{1}{n^2} \mathbb{E}_{\lambda}\left[\underbrace{\sum_{\mathcal{D}_{train}^2}\mathcal{L} \left( \mathbf{\tilde y}_{i,j,\lambda} ,f(\mathbf{\tilde  x}_{i,j,\lambda})\right)}_{L_{mixup}}\right]. 
   \end{align*}

\end{definition}

In other words, \emph{Mixup} encourages the model $f$ to associate linearly interpolated inputs with the corresponding linearly interpolated outputs~\citep{zhang2017mixup}. The positive effect of this linear behavior in between samples questioned several authors who aimed at explaining theoretically and empirically \emph{Mixup}. \citet{carratino2020mixup} shows that Mixup can be interpreted as the combination of a data transformation and a data perturbation. A first transform shrinks both inputs and outputs towards their mean. The second transform applies a zero mean perturbation. The proof is given by reformulating the Mixup loss. \citet{Lipmixup} highlight that \emph{Mixup} impacts the Lipschitz constant $L$ of the gradient of the network.

\textbf{Improvements over mixup:}
In other works, authors propose to improve \emph{Mixup} using various approaches. For example in~\citep{remix}, the idea is to use different $\lambda_x,\lambda_y$ to mix the input and the outputs, in~\citep{automix,mixmo,cutmixorig}, the authors explore using other (i.e. nonlinear) interpolation methods, in~\citep{batchmixup,kmixup,stackmix} the authors extend the mixing to more than two elements.


\textbf{Our proposed approach:}
In this paper, we aim at avoiding the phenomenon described as \emph{manifold intrusion}, and introduced in~\citep{guo2019mixup}. This phenomenon is depicted in Figure~\ref{fig:method} on the right, where we see that virtual samples created through mixup between distant red samples lie outside the manifold domain for the red class. As we do not have access to the underlying manifold domains when we train a model, the rationale of our contribution is to favor interpolations between samples that are close enough in the input domain. Where the method described in~\citep{guo2019mixup} learns which interpolations should be kept through training, we advocate in this paper for a purely geometric approach where a decreasing weight is applied when computing the loss depending on the distance between interpolated samples.



\section{Mixup in dimension 1}\label{theomixup}

Let us consider the simple case where our model $f$ is defined on $\mathbb{R}$. Without loss of generality, let us consider that the training set $\mathcal{D}_{train} = \{x_i, y_i\}$ is ordered by increasing input, i.e, $x_i \leq x_{i+1}$.


For a given $\tilde x$, \emph{Mixup}'s loss implies that the output $f^*(\tilde x)$ of the model is determined by the set $\mathcal{E}(\tilde x)$ of all convex combinations that can be obtain $\tilde x$ from two training inputs $x_i$ and $x_j$: $\mathcal{E}(\tilde x)= \{ i,j,\lambda_{i,j} \vert \tilde x = \lambda_{ij} x_i+(1-\lambda_{ij})\ x_j\}$. 
 It is clear that for any $\tilde x \in [x_0,x_{n-1}]$, $\mathcal{E}(\tilde x)$ is non empty and finite. In practice, the distribution of $\lambda$ can be uniform  ~\citep{zhang2017mixup,manifold}  $\lambda \sim Beta(\alpha = 1, \beta = 1) = \mathcal{U}(0,1)$. In this case, we show that the output $f^*(\tilde x)$ of an input $x \in [x_0,x_n]$ is the barycenter of the target values corresponding to the points of $\mathcal{E}(\tilde x)$.
\begin{lemma}
   $\forall \tilde x \in [x_0,x_{n-1}]$,
    \begin{align}
    f^*(\tilde x) = \frac{1}{card(\mathcal{E}(\tilde x))} \sum_{(i,j,\lambda_{i,j}) \in {\mathcal{E}(\tilde x)}}\lambda_{i,j}y_i+(1-\lambda_{i,j})y_j.
    \label{eq:solution}
\end{align}
\end{lemma}
\begin{proof}
Let $ \tilde x \in [x_0,x_{n-1}]$ and $0\leq\lambda \leq 1$. For a given triplet $(i,j,\lambda) \in \mathcal{E}(\tilde x)$. We have $\mathbb{E}[\mathcal{L}(y_i,j,\lambda_{i,j},f^*(\tilde x)) \vert \tilde x, i,j,\lambda_{ij} ]= \mathcal{L}(y_i,j,\lambda_{i,j},f^*(\tilde x))$ as the value of $y_i,j,\lambda_{i,j}$ and $\tilde x$ are known. Then we minimize the error for all $y_i,j,\lambda_{i,j}$ given by $\mathcal{E}(\tilde x)$. Then the value of $f^*(x)$ is only determined by the sum of the losses over $\mathcal{E}(\tilde x)$
since the elements of $\mathcal{E}(\tilde x)$ are equally probable (distributions of $i,j,\lambda$ are uniform). 
\begin{align}
    \mathbb{E}[\mathcal{L}(f^*(\tilde x),y_{i,j,\lambda_{i,j}}) &= \sum_{\mathcal{E}(\tilde x)} \mathbb{E}[\mathcal{L}(f^*(\tilde x),y_{i,j,\lambda_{i,j}}) \vert \tilde x, i,j,\lambda_{ij}]  \nonumber\\
    &= \sum_{\mathcal{E}(\tilde x)}\mathcal{L}(f^*(\tilde x),y_{i,j,\lambda_{i,j})})
    \label{eq:summixup}
\end{align}
We assume $\mathcal{L}$ to be either the cross entropy or the squared L2 loss. In either case, by nulling the derivative of Equation~\eqref{eq:summixup} w.r.t. the value $f^*(\tilde{x})$, we get: 
\begin{align*}
    f^*(\tilde x) = \frac{1} {card(\mathcal{E}(\tilde x )))}\sum_{\mathcal{ E}(\tilde x)} y_{i,j,\lambda_{i,j}} 
\end{align*}
\end{proof}
A consequence of this lemma is the following theorem:
\begin{theorem}
     The function $f^*$ that minimizes the loss on the training set is piecewise linear on $[x_0,x_{n-1}]$, linear on each segment $[x_i,x_{i+1}]$ and defined by Equation \eqref{eq:solution}.
\end{theorem}

When $\tilde{x}$ varies in $[x_i,x_{i+1}]$, the set of possible combinations (between training samples) leading to $\tilde{x}$ does not change, only the corresponding coefficients $\lambda$ vary linearly. Since the expression of Equation~\eqref{eq:solution} is linear in each of those coefficients, $f^*$ is itself linear as a function of $\tilde x$. The set of possible combinations will change whenever $\tilde x$ switches to another interval, e.g. $[x_{i-1},x_i]$. In this case new combinations are possible and others may disappear, leading to another linear function. $f^*$ is still continuous everywhere because new or disappearing combinations are associated either to $\lambda = 0$ or $\lambda=1$ for $\tilde{x} = x_{j}$ and $ j \in \{1,\cdots,n\}$.

In practice inferring a function $f^*$ that minimizes that the Mixup Criterion is usually not desired in machine learning, and one looks for $f$ with a sufficiently small loss to have a regularizing effect. Indeed $f^*$ is not likely to generalize well. Still, we note that it tends to an average of convex combinations and thus leads to a model with a low variance. 

\section{Local Mixup}


\subsection{Locality graphs}
Consider a (training) dataset $\mathcal{D}$ made of pairs $(\mathbf{x}, \mathbf{y})$. We propose to build a graph from $\mathcal{D}$ as follows. We define $G_{\mathcal{D}} = \langle V, \mathbf{W} \rangle$ where $V = \{\mathbf{x} \ \vert \ \exists \mathbf{y}, (\mathbf{x},\mathbf{y}) \in \mathcal{D}\}$. The symmetric real matrix $\mathbf{W}$ is based on $D$, where $D$ is the pairwise distance matrix $D[i,j] = d_\mathcal{X}(\mathbf{x}_i, \mathbf{x}_j)$.

In this work, we consider various ways to obtain $\mathbf{W}$, but the rationale is always the same: to obtain a similarity matrix where large weights correspond to closest pairs of samples. Namely, we consider $K$-nearest neighbors graphs, where we set to 1 weights of target vertices corresponding to the $K$ closest samples for a given source vertex and 0 otherwise; thresholded graphs where $\mathbf{W}[i,j] = \phi(D[i, j])$ and $\phi(d) = \mathbf{1}_{d \leq \varepsilon}$; smooth decreasing exponential graphs where $\mathbf{W}[i, j] = \exp( -\alpha \mathbf{D}[i,j])$. The loss is then weighted using $\mathbf{W}$: 
 \begin{align}
L_{\text{local mixup}} = \sum_{\mathcal{D}_{train}^2}\mathbf{W}[i, j]\mathcal{L}\left( \mathbf{\tilde y}_{i,j,\lambda} ,f(\mathbf{\tilde  x}_{i,j,\lambda})\right).
\end{align} 
For computational cost considerations, we compute a graph for each batch (random subset) of samples during stochastic gradient descent. As such, the weights associating two samples can vary depending on the chosen graph and random batch.

In the extreme case where some weights are 0, the corresponding virtual samples are discarded during gradient descent, resulting in only considering local interpolations of samples, hence the name \emph{Local Mixup}.

\subsection{Low dimension}

In this section, we are interested in proving that \emph{Local Mixup} allows to tune a trade-off between bias and variance on trained models. For this purpose, we simplify the problem to dimension 1 and only consider $K$-nearest neighbor graphs.

In this case, note that varying $K$ can create a range of settings where $K=0$ boils down to vanilla training and $K \geq n$ where $n$ is the number of training samples boils down to classical \emph{Mixup}.

\subsubsection{Local Mixup and the bias/variance trade-off}

Let us first recall the definitions of the bias and variance in the context of a machine learning problem.

\begin{definition}[Bias and Variance] Let us consider a training set $\mathcal{D}_{train}$ and a function $f$ from $\mathcal{X}$ to $\mathcal{Y}$. We define Bias and Variance as follow:
\begin{itemize}
    \item \emph{Bias}: $Bias(f)^2 = \mathbb{E}_{train}[(f(x)-y)^2]$.
    \item \emph{Variance}: $Var(f)= \mathbb{E}_{train} [(f-\mathbb{E}_{train}[f])^2]$.
\end{itemize}
\end{definition}

We consider two settings. In the first one, the input domain $\mathbb{Z}/n\mathbb{Z}$ is periodic and thus the number of samples is finite. In the second one, the input domain $\mathbb{Z}$ is infinite and outputs are independent and identically (i.i.d) generated using a random variable.
\newline 

\textbf{Periodic setting}

Let us consider that the training set $\mathcal{D}_{train}$ is made of pairs $(x,y)$, where $\{x \ \vert \ \exists y, (x,y) \in \mathcal{D}_{train}\} = \mathbb{Z}/n\mathbb{Z}$. We also consider $d_\mathcal{X}(x, x') = |x - x'| \in \{0,\cdots,n-1\}$.

In this case, we can write explicit formulations of $f^*_{K}$, the function that minimizes the \emph{Local Mixup} criterion for $K$-nearest neighbors graphs. Following similar arguments to those used to obtain Equation~\eqref{eq:solution}: for a given $x_i$ we know that the optimal value for $f^*_K(x_i)$ would be an average of the the $\tilde y$ that correspond to the possible interpolations.
we obtain:
\begin{align}
    \forall x_i \in \mathbb{Z}/n\mathbb{Z}, f^*_K(x_i) &= \frac{1}{K(K+3)/2}(2 K y_i +S_K(x_i)) \label{eq:fkmixup},
\end{align} 
where $S_K(x_i)$ is defined recursively as follows:
\begin{equation}
    S_{K+1} = \left\{\begin{array}{ll}0& \text{if } K = 0\\S_{K}(x_i) + A_{K+1}(x_i)& \forall K \geq 1 \end{array}\right..
    \label{eq:skexpression}
\end{equation}
and:
\begin{align*}
A_K(x_i) &= \frac{1}{K} \sum_{k=1}^{K-1} (K-k)\cdot y_{i-k} + k\cdot y_{i+K-k}.
\end{align*}
On Figure~\ref{fig:illustrationfk} we depicted for a given $x_i$ the different interpolations and $\tilde y$ that contribute to $f_K(x_i)$. In blue the interpolation between $x_i$ and its direct neighbors, in red the interpolation between points other than $x_i$ that happen to intersect $x_i$. As we increase $K$, the influence of $S_K$ (red points) increases. 

\begin{figure}
     \begin{tikzpicture}
 \begin{axis}[
xticklabels={$\tiny{x_{i-2}}$,$\tiny{x_{i-1}}$,$\tiny{x_{i}}$,$\tiny{x_{i+1}}$,$\tiny{x_{i+2}}$},
xtick={1,2,3,4,5},
xmajorgrids,
ymajorgrids,
yticklabels={,,},
ymin=8, ymax=26,
width=5cm, height=4cm]
\addplot[only marks,blue] coordinates { (1,10)(2,20)(3,10)(4,20)(5,25)};
\addplot[only marks,red] coordinates { (3,20)};
\addplot[mark=none, red, dashed,name path=indirect] coordinates {(2,20) (4,20)};
\addplot[mark=none, blue,name path=direct] coordinates {(2,20) (3,10) (4,20)};
\addplot[blue!20] fill between[of=indirect and direct];

\end{axis}


\begin{axis}[
xshift=4cm,
xticklabels={$\tiny{x_{i-2}}$,$\tiny{x_{i-1}}$,$\tiny{x_{i}}$,$\tiny{x_{i+1}}$,$\tiny{x_{i+2}}$},
xtick={1,2,3,4,5},
xmajorgrids,
ymajorgrids,
yticklabels={,,},
ymin=8, ymax=26,
width=5cm, height=4cm]

\addplot[only marks,blue] coordinates { (1,10)(2,20)(3,10)(4,20)(5,25)};
\addplot[only marks,red] coordinates { (3,20)};
\addplot[only marks,red] coordinates { (3,16.7)};
\addplot[only marks,red] coordinates { (3,21.7)};
\addplot[mark=none, red, dashed,name path=indirect] coordinates {(2,20) (4,20)};
\addplot[mark=none, blue,name path=direct] coordinates {(2,20) (3,10) (4,20)};
\addplot[mark=none, blue,name path=direct2] coordinates {(1,10) (3,10) (4,20)};
\addplot[mark=none, blue,name path=direct3] coordinates {(2,20) (3,10) (5,25)};
\addplot[mark=none, red, dashed,name path=indirect2] coordinates {(1,10) (4,20)};
\addplot[mark=none, red, dashed,name path=indirect3] coordinates {(2,20) (5,25)};
\addplot[blue!20] fill between[of=indirect and direct];
\addplot[blue!20] fill between[of=indirect2 and direct2];
\addplot[blue!20] fill between[of=indirect3 and direct3];
\end{axis}
\end{tikzpicture} 

\caption{We depict here the terms of $f^*_K(x_i)$ given by Eq~\eqref{eq:fkmixup} for different K. In blue the interpolations corresponding to $2Ky_i$ and in red the terms of the sum $S_K$. On the right, $K =2$ and on the left $K = 3$. }
\label{fig:illustrationfk}
\end{figure}
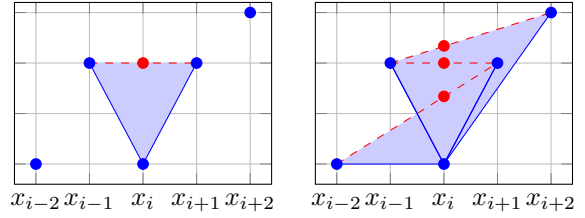


We obtain the following Lemma, showing that the expected value of $f^*_K$ is invariant with respect to $K$:
\begin{lemma}\label{lemmaE(fk)}[Expected value of $f^*_k$] For any $K$, the expected value of $f^*_K$ is 
    \begin{align}
    \mathbb{E}_{train}[f^*_K] = \mathbb{E}_{train}[y].
\end{align}
\end{lemma}
\begin{proof}
\begin{equation*}
    \mathbb{E}_{train}[f^*_K] = \frac{1}{n}\sum_{i=1}^n f^*_K(x_i)\\
\end{equation*}
\begin{equation*}
    =\frac{2}{nK(K+3)}(2nK\mathbb{E}_{train}[y]+ \sum_{i=1}^n \sum_{k=1}^{K}A_k(x_i)). \\
\end{equation*}
and using the fact that $y_{i+n}= y_i$:
\begin{align*}
    &\sum_{i=1}^n \sum_{k=1}^{K}A_k(x_i) = \sum_{i=1}^n \sum_{k=1}^{K}\sum_{l=1}^{k-1}\frac{k-l}{k}y_{i-l}+\frac{l}{k}y_{i+k-l}\\
    &= \sum_{i=1}^n y_{i}  \sum_{k=1}^{K}\sum_{l=1}^{k-1}1
    = n\mathbb{E}_{train}[y] \frac{K(K-1)}{2}.
\end{align*}
then $\mathbb{E}_{train}[f^*_K] = \mathbb{E}_{train}[y]$.
\end{proof}

We obtain the following theorem:

\begin{theorem}[Convergence of $f^*_K$ in the periodic setting]
As $K$ grows, it holds that:
    \begin{align}
     &\forall x_i \in \mathbb{Z}/n\mathbb{Z}, \nonumber\\
     &f_K^*(x_i) \to \mathbb{E}_{D_{train}}[y],\\
     &Bias^2(f^*_K) \to \mathbb{E}_{train}[(y_i-\mathbb{E}_{train}[y])^2],\\
     &Var(f^*_K) =  \mathbb{E}_{train}[\left(f^*_K(x_i)-\mathbb{E}_{train}[f^*_K(x_i]\right)^2] \to 0, \\
     &Var(f^*_K) \text{ is eventually nonincreasing}. \nonumber
    \end{align}
\end{theorem}
\begin{proof}
We can explicitly write the limit of $S_K$. We first prove this lemma (the proof is omitted here but available as supplementary material):
\begin{lemma}
Let $K = Mn + r$, $M \in \mathbb{N}^*$ and $0\leq r<n-1$. We assume $\mathbb{E}_{train}(y) \geq 0$, then:
    \begin{align}
         &(M+1)n \cdot \mathbb{E}_{train}(y) + \eta \  \geq A_K \geq Mn \cdot\mathbb{E}_{train}(y) -\eta,
    \end{align}
with $\eta=\mathcal{O}(K\mathbb{E}_{train}(y))$.
\end{lemma} 

Then combined with Equation~\eqref{eq:skexpression} we can demonstrate the convergence of the sum $S_K$ and find its limit:
\begin{corollary}
    For $K = NM \to \infty$
    \begin{align}
        S_K &\to \frac{1}{2}\sum_{i=1}^n y_iM^2n  =\frac{1}{2} \mathbb{E}_{train}(y)K^2.
    \end{align}
\end{corollary}

As a result, given Equation~\eqref{eq:fkmixup}, the limit of $f^*_K$ is $\mathbb{E}_{train}(y)$.

To prove the monotonicity of the variance we want to show: $Var(f^*_{K+1}) \leq Var(f^*_K)$ for $K$ large enough. We use the König-Huygens theorem and Lemma~\ref{lemmaE(fk)} to compute the difference between the two variances:
\begin{align*}
    & Var(f_{K+1}) - Var(f_K) \\
    &=\mathbb{E}_{D_{train}}[\left(f_{K+1}(x)\right)^2] -  \mathbb{E}_{D_{train}}[\left(f_K(x)\right)^2] \\
    &= \mathbb{E}_{D_{train}}[\left(f_{K+1}(x)\right)^2 -  \left(f_K(x)\right)^2].
\end{align*}
We then show that for any $x\in [x_0,x_{n-1}]$ and $K$ large enough, $\left(f_{K+1}(x)\right)^2 \leq \left(f_{K}(x)\right)^2$. To do so we get an asymptotic equivalent: 
\begin{align*}
&\left(f_{K+1}(x)\right)^2 - \left(f_{K}(x)\right)^2 \sim  -\frac{ K }{C} \cdot E_{train}^2[y],\\
& \text{where C is a positive constant}.\qedhere
\end{align*}\end{proof}

This theorem states two main results: 1) in the case of \emph{Mixup} the function that minimize the loss $f^*$ has zero variance and converges to $\mathbb{E}_{train}[y]$. 2). Eventually the variance of the function that minimizes the \emph{Local Mixup} criterion is decreasing, showing that the proposed \emph{Local Mixup} can indeed tune the trade-off between the bias and variance.

\textbf{i.i.d random output setting}

Let us now consider that the training set is made of inputs $\{x \ \vert \ \exists y, (x,y) \in \mathcal{D}_{train}\} = \mathbb{Z}$ and $y_i$ are i.i.d. according to a random variable $R$ of variance $\sigma^2$.

\begin{theorem}[]
For a signal with i.i.d outputs, the variance is eventually bounded by:
\begin{align}
     \frac{4^2\sigma^2}{K^2} &\leq Var(f_K(x_i)) \leq \frac{ 8\sigma^2}{K}.
\end{align}

\end{theorem}

\begin{proof}
Let us choose $x_i$ and $K>1$. First observe that $f^*_K(x_i)$ is a sum of random variables. 
We rewrite $S_K$ with the coefficients $a_k^{K}= \sum_{l=k+1}^K \frac{l-k}{l}$: $S_K~=~\sum_{k =1}^{K-1} (y_{i-k}+y_{i+k}) a_k^{K}$. We obtain:
\begin{equation*}
    Var(f^*_K(x_i)) = Var\left( \frac{2\cdot(2Ky_i + S_K)}{K(K+3)}. \right)
    \end{equation*}
leading to:
\begin{align*}
    Var(f^*_K(x_i))= 4^2\left( \frac{K}{K(K+3)} \right)^2Var(y_i) \\ + \sum_{k=1}^{K-1}\left(\frac{2a_k^{(K)}}{K(K+3)}\right)^2 (Var(y_{i-k}) + Var(y_{i+k})).
\end{align*}
We use the fact that $\frac{1}{K}\leq  a_k^{K} \leq K$.

Then when $K \to \infty$:
\begin{align*}
   \frac{4^2\sigma^2}{K^2} &\leq Var(f_K(x_i)) \leq \frac{ 8\sigma^2}{K}.\qedhere
\end{align*} 
\end{proof}

\subsubsection{Invariance of linear models}

Interestingly, we can show that both \emph{Mixup} and \emph{Local Mixup} lead to the same optimal linear models, as stated in the following theorem:

\begin{theorem}
    For a linear model: $f(x) = ax + b, \ a,b \in \mathbb{R}$, the function $f^*$ that minimizes the loss of \emph{Mixup} and \emph{Local Mixup} is the same. 
\end{theorem} 
\begin{proof}
    For mixup, we showed with Equation~\eqref{eq:solution} the function $f^*$ is a piecewise linear function. The same equation applies for \emph{Local Mixup} except that the set $E_x$ is smaller for \emph{Local Mixup} as the number of endpoints is restricted.
    As a piecewise linear function, linear on each segment $[x_i,x_{i+1}]$: $f^*$ can be written as $ f^* = a_i x  + b_i$ where each $(a_i,b_i)$ are defined on $[x_i,x_{i+1}]$.
    Let us consider $\mathcal{F}$ to be restricted to linear functions, then the coefficients $a$, $b$ are the averages of the $(a_i,b_i)$.
\end{proof}

\subsection{High Dimension and Lipschitz constraint}\label{Lipschistsection}
The proofs given in low dimension have some limitations. Basically, the averaging effect happens since any point $x$ within the interval $[x_1,x_n]$ can be written as at least one convex combination of pairs from the training set. Contradictions may occur as illustrated above when several combinations corresponds to $x$. In higher dimension such explicit contradictions are not necessarily expected. Still, we show that \emph{Local Mixup} has an impact on the Lipschitz constant of the networks.

First recall the definition of a $q$-Lipschitz function: 

\begin{definition}[Lipschitz Continuous and Lipschitz Constant]
    Given two metric spaces $(\mathcal{X},d_X), (\mathcal{Y},d_\mathcal{Y})$ and a function $f: \mathcal{X} \to \mathcal{Y}$, $f$ is Lipschitz continuous if there exists a real constant $q \geq 0$ s.t for all $x_i$ and $x_j$ in $\mathcal{X}$, 
    \begin{align}
         d_\mathcal{Y}\left(f(x_i),f(x_j)\right) \leq q d_\mathcal{X}(x_i,x_j).
    \end{align}
     If $f$ is $q$-Lipschitz continuous, we define the optimal Lipschitz constant $Q_{sup}$ as
     \begin{align}
          Q_{sup} = \sup_{x_i,x_j \in X, x_i \neq x_j} \frac{d_\mathcal{Y}\left(f(x_i),f(x_j)\right)}{d_\mathcal{X}(x_i,x_j)}.
     \end{align}
\end{definition}
 
For simplicity, let us consider a classification problem where $d_\mathcal{Y}$ is 0 if the two considered samples are of the same class and 1 otherwise.

Then the training set imposes a lower bound on the optimal Lipschitz constant:
\begin{align}
    Q_{sup} \geq \underbrace{\left(\min_{\mathbf{x}_i,\mathbf{x}_j \in \mathcal{D}, y_i \neq y_j} d_\mathcal{X}(x_i,x_j) \right)^{-1}}_{Q(D)}.
\end{align}

For \emph{Mixup} and \emph{Local Mixup}, the virtual samples increase the size of the training set, resulting in stronger constraints on the optimal Lipschitz constant.

In more details, consider the case of a thresholded graph with parameter $\varepsilon$ when using \emph{Local Mixup}. In this case, the increased training set for each class $\mathbf{y}$ can be written as $S_\varepsilon(\mathbf{y}) = \{ \lambda \mathbf{x}_i + (1-\lambda) \mathbf{x}_j \ \vert \  0\leq \lambda\leq 1 \ \mathbf{y}_i = \mathbf{y}_j = \mathbf{y}, d_{\mathcal{X}}(\mathbf{x}_i, \mathbf{x}_j) \leq \varepsilon  \},$ the set of all segments constructed from two samples that are close enough in the input domain and sharing the same label $\mathbf{y}$. We then obtain the following theorem:

\begin{theorem}\label{theoremLip}
    The lower bound $Q(D)$ is increasing with $\varepsilon$.
\end{theorem}
\begin{proof}
We directly use the inclusion $S_\varepsilon(\mathbf{y})\subset S_{\varepsilon'}(\mathbf{y}), \forall \varepsilon \leq \varepsilon'$.
\end{proof}

We shall show in the experiments that $\varepsilon$ can indeed impact $Q(D)$ on standard vision datasets.

\section{Experiments}
\subsection{Low dimension}
As stated in the introduction and~\citep{guo2019mixup}, \emph{Mixup} leads to interpolations that may be misleading for the model. To illustrate this effect, we consider a 2d toy dataset of two coiling spirals where such interpolations occur frequently. The two coiling spirals is a binary classification dataset: each spiral corresponds to a different class. We expect to retrieve better performance for \emph{Local Mixup} compared to \emph{Mixup}: local interpolations are likely to stay in the same spiral and therefore avoid manifold intrusion. For this experiment we use a thresholded graph with parameter $\varepsilon$.

To carry out this experiment, we generate 1000 samples for each class and add a Gaussian noise with standard deviation $\sigma = 1.5$ (controlling the spirals' thickness). A typical draw is depicted in Figure~\ref{fig:spirals}. We use a large value of $\sigma$ to avoid trivial solutions to the problem. Once the dataset is generated we split it randomly into two parts: a training set containing 80\% of the samples and a test set containing the remaining 20\% (used to compute the error rates).
\begin{figure}
    \centering
    \includegraphics[width = 8.4cm,height =4cm,trim={3cm 1.4cm 2cm 1.2cm},clip]{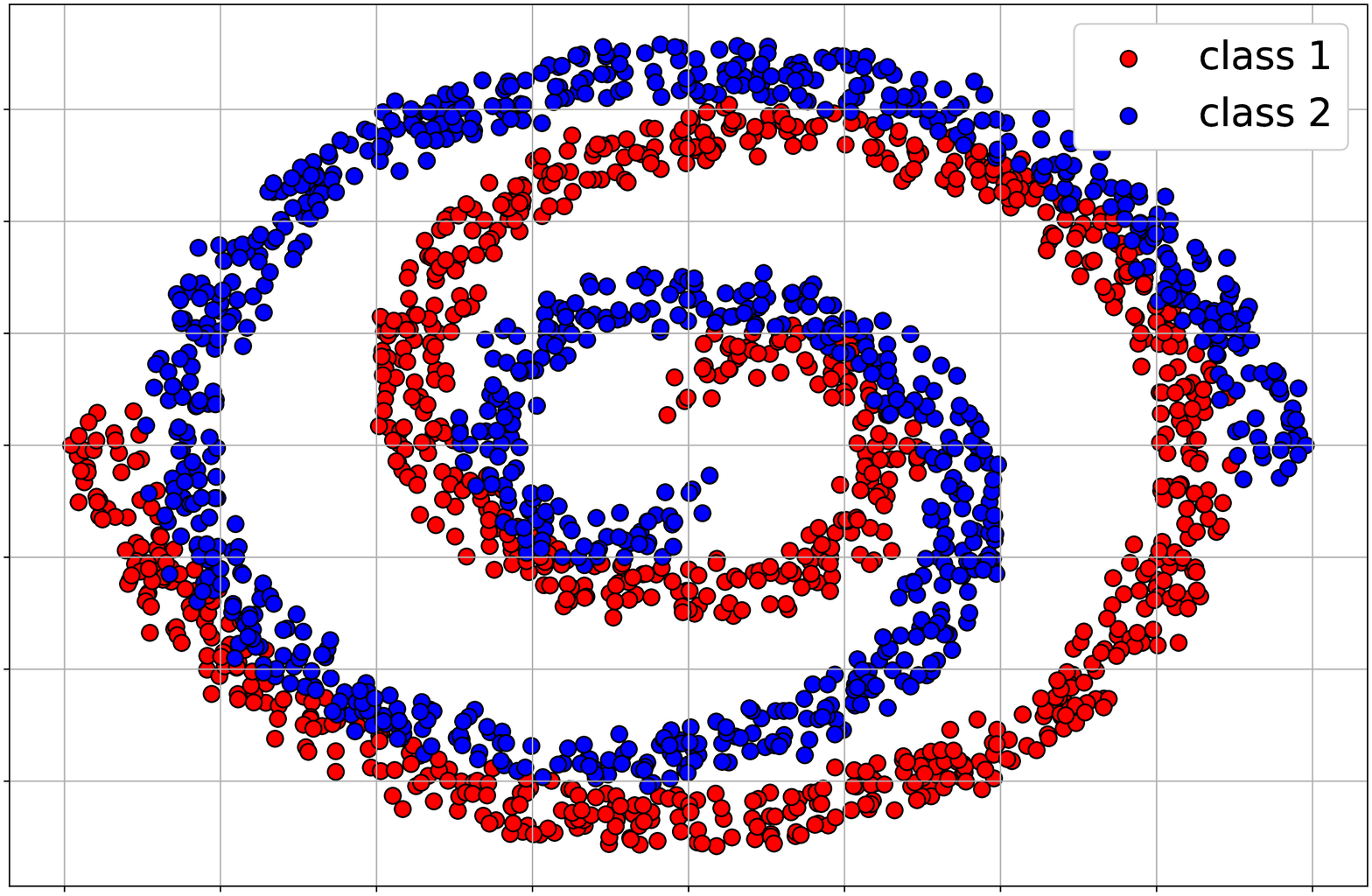}
    \caption{Illustration of the two coiling spiral dataset with 1000 samples per class and $\sigma = 1.5$.}
    \label{fig:spirals}
\end{figure}
We then use a fully connected neural network made of two hidden layers with 100 neurons and use the ReLU function as non linearity. We average the test errors over 1000 runs. 
For small values of $\varepsilon$ many weights of the graph are zero and thus the corresponding interpolations are disregarded into the loss. This means that for a given batch only a small proportion of samples are regarded to compute the loss. Without any correction, different values of $\varepsilon$ lead to different batch sizes. To avoid side effects, we vary the batch size so that in average the same number of samples are used to update the loss.

To select an appropriate value of $\varepsilon$, we first looked at the distribution of distances between pairs of inputs in the training set. This distribution is depicted in Figure~\ref{fig:histo}. We observe that the distribution is relatively uniform between 0 and 4, and as such in our experiments we vary $\varepsilon$ between 0 and 4 using steps of 0.5.

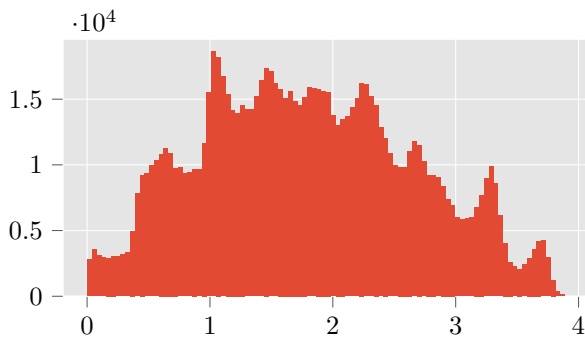
\begin{figure}
    \centering
    \begin{tikzpicture}

\definecolor{color0}{rgb}{0.886274509803922,0.290196078431373,0.2}\pgfplotsset{%
    width=.54\textwidth,
    height=.25\textwidth
}

\begin{axis}[
axis background/.style={fill=white!89.8039215686275!black},
axis line style={white},
tick align=outside,
tick pos=left,
x grid style={white},
xmajorgrids,
xmin=-0.19454230702101, xmax=4.08538844744121,
xtick style={color=white!33.3333333333333!black},
y grid style={white},
ymajorgrids,
ymin=0, ymax=19555.2,
ytick style={color=white!33.3333333333333!black},
width=.5\textwidth,
height=5cm
]
\draw[draw=none,fill=color0,very thin] (axis cs:0,0) rectangle (axis cs:0.038908461404202,2780);
\draw[draw=none,fill=color0,very thin] (axis cs:0.038908461404202,0) rectangle (axis cs:0.077816922808404,3560);
\draw[draw=none,fill=color0,very thin] (axis cs:0.077816922808404,0) rectangle (axis cs:0.116725384212606,3100);
\draw[draw=none,fill=color0,very thin] (axis cs:0.116725384212606,0) rectangle (axis cs:0.155633845616808,2960);
\draw[draw=none,fill=color0,very thin] (axis cs:0.155633845616808,0) rectangle (axis cs:0.19454230702101,2922);
\draw[draw=none,fill=color0,very thin] (axis cs:0.19454230702101,0) rectangle (axis cs:0.233450768425212,3054);
\draw[draw=none,fill=color0,very thin] (axis cs:0.233450768425212,0) rectangle (axis cs:0.272359229829414,3036);
\draw[draw=none,fill=color0,very thin] (axis cs:0.272359229829414,0) rectangle (axis cs:0.311267691233616,3178);
\draw[draw=none,fill=color0,very thin] (axis cs:0.311267691233616,0) rectangle (axis cs:0.350176152637818,3340);
\draw[draw=none,fill=color0,very thin] (axis cs:0.350176152637818,0) rectangle (axis cs:0.38908461404202,4954);
\draw[draw=none,fill=color0,very thin] (axis cs:0.38908461404202,0) rectangle (axis cs:0.427993075446222,7804);
\draw[draw=none,fill=color0,very thin] (axis cs:0.427993075446222,0) rectangle (axis cs:0.466901536850424,9240);
\draw[draw=none,fill=color0,very thin] (axis cs:0.466901536850424,0) rectangle (axis cs:0.505809998254626,9332);
\draw[draw=none,fill=color0,very thin] (axis cs:0.505809998254626,0) rectangle (axis cs:0.544718459658828,9946);
\draw[draw=none,fill=color0,very thin] (axis cs:0.544718459658828,0) rectangle (axis cs:0.58362692106303,10348);
\draw[draw=none,fill=color0,very thin] (axis cs:0.58362692106303,0) rectangle (axis cs:0.622535382467232,10828);
\draw[draw=none,fill=color0,very thin] (axis cs:0.622535382467232,0) rectangle (axis cs:0.661443843871434,11280);
\draw[draw=none,fill=color0,very thin] (axis cs:0.661443843871434,0) rectangle (axis cs:0.700352305275636,10916);
\draw[draw=none,fill=color0,very thin] (axis cs:0.700352305275636,0) rectangle (axis cs:0.739260766679838,9776);
\draw[draw=none,fill=color0,very thin] (axis cs:0.739260766679838,0) rectangle (axis cs:0.77816922808404,9782);
\draw[draw=none,fill=color0,very thin] (axis cs:0.77816922808404,0) rectangle (axis cs:0.817077689488242,9342);
\draw[draw=none,fill=color0,very thin] (axis cs:0.817077689488242,0) rectangle (axis cs:0.855986150892444,9412);
\draw[draw=none,fill=color0,very thin] (axis cs:0.855986150892444,0) rectangle (axis cs:0.894894612296646,9682);
\draw[draw=none,fill=color0,very thin] (axis cs:0.894894612296645,0) rectangle (axis cs:0.933803073700847,9650);
\draw[draw=none,fill=color0,very thin] (axis cs:0.933803073700848,0) rectangle (axis cs:0.97271153510505,11648);
\draw[draw=none,fill=color0,very thin] (axis cs:0.972711535105049,0) rectangle (axis cs:1.01161999650925,15508);
\draw[draw=none,fill=color0,very thin] (axis cs:1.01161999650925,0) rectangle (axis cs:1.05052845791345,18624);
\draw[draw=none,fill=color0,very thin] (axis cs:1.05052845791345,0) rectangle (axis cs:1.08943691931766,18222);
\draw[draw=none,fill=color0,very thin] (axis cs:1.08943691931766,0) rectangle (axis cs:1.12834538072186,16724);
\draw[draw=none,fill=color0,very thin] (axis cs:1.12834538072186,0) rectangle (axis cs:1.16725384212606,15408);
\draw[draw=none,fill=color0,very thin] (axis cs:1.16725384212606,0) rectangle (axis cs:1.20616230353026,14178);
\draw[draw=none,fill=color0,very thin] (axis cs:1.20616230353026,0) rectangle (axis cs:1.24507076493446,13964);
\draw[draw=none,fill=color0,very thin] (axis cs:1.24507076493446,0) rectangle (axis cs:1.28397922633867,14558);
\draw[draw=none,fill=color0,very thin] (axis cs:1.28397922633867,0) rectangle (axis cs:1.32288768774287,14200);
\draw[draw=none,fill=color0,very thin] (axis cs:1.32288768774287,0) rectangle (axis cs:1.36179614914707,14204);
\draw[draw=none,fill=color0,very thin] (axis cs:1.36179614914707,0) rectangle (axis cs:1.40070461055127,15202);
\draw[draw=none,fill=color0,very thin] (axis cs:1.40070461055127,0) rectangle (axis cs:1.43961307195547,16412);
\draw[draw=none,fill=color0,very thin] (axis cs:1.43961307195547,0) rectangle (axis cs:1.47852153335968,17374);
\draw[draw=none,fill=color0,very thin] (axis cs:1.47852153335968,0) rectangle (axis cs:1.51742999476388,17122);
\draw[draw=none,fill=color0,very thin] (axis cs:1.51742999476388,0) rectangle (axis cs:1.55633845616808,16202);
\draw[draw=none,fill=color0,very thin] (axis cs:1.55633845616808,0) rectangle (axis cs:1.59524691757228,15766);
\draw[draw=none,fill=color0,very thin] (axis cs:1.59524691757228,0) rectangle (axis cs:1.63415537897648,15094);
\draw[draw=none,fill=color0,very thin] (axis cs:1.63415537897648,0) rectangle (axis cs:1.67306384038069,15634);
\draw[draw=none,fill=color0,very thin] (axis cs:1.67306384038069,0) rectangle (axis cs:1.71197230178489,14846);
\draw[draw=none,fill=color0,very thin] (axis cs:1.71197230178489,0) rectangle (axis cs:1.75088076318909,14562);
\draw[draw=none,fill=color0,very thin] (axis cs:1.75088076318909,0) rectangle (axis cs:1.78978922459329,15184);
\draw[draw=none,fill=color0,very thin] (axis cs:1.78978922459329,0) rectangle (axis cs:1.82869768599749,15888);
\draw[draw=none,fill=color0,very thin] (axis cs:1.82869768599749,0) rectangle (axis cs:1.86760614740169,15818);
\draw[draw=none,fill=color0,very thin] (axis cs:1.8676061474017,0) rectangle (axis cs:1.9065146088059,15780);
\draw[draw=none,fill=color0,very thin] (axis cs:1.9065146088059,0) rectangle (axis cs:1.9454230702101,15638);
\draw[draw=none,fill=color0,very thin] (axis cs:1.9454230702101,0) rectangle (axis cs:1.9843315316143,15560);
\draw[draw=none,fill=color0,very thin] (axis cs:1.9843315316143,0) rectangle (axis cs:2.0232399930185,13750);
\draw[draw=none,fill=color0,very thin] (axis cs:2.0232399930185,0) rectangle (axis cs:2.0621484544227,13038);
\draw[draw=none,fill=color0,very thin] (axis cs:2.0621484544227,0) rectangle (axis cs:2.10105691582691,13500);
\draw[draw=none,fill=color0,very thin] (axis cs:2.10105691582691,0) rectangle (axis cs:2.13996537723111,13664);
\draw[draw=none,fill=color0,very thin] (axis cs:2.13996537723111,0) rectangle (axis cs:2.17887383863531,14370);
\draw[draw=none,fill=color0,very thin] (axis cs:2.17887383863531,0) rectangle (axis cs:2.21778230003951,15054);
\draw[draw=none,fill=color0,very thin] (axis cs:2.21778230003951,0) rectangle (axis cs:2.25669076144372,16226);
\draw[draw=none,fill=color0,very thin] (axis cs:2.25669076144372,0) rectangle (axis cs:2.29559922284792,16110);
\draw[draw=none,fill=color0,very thin] (axis cs:2.29559922284792,0) rectangle (axis cs:2.33450768425212,15198);
\draw[draw=none,fill=color0,very thin] (axis cs:2.33450768425212,0) rectangle (axis cs:2.37341614565632,14568);
\draw[draw=none,fill=color0,very thin] (axis cs:2.37341614565632,0) rectangle (axis cs:2.41232460706052,12846);
\draw[draw=none,fill=color0,very thin] (axis cs:2.41232460706052,0) rectangle (axis cs:2.45123306846473,12028);
\draw[draw=none,fill=color0,very thin] (axis cs:2.45123306846472,0) rectangle (axis cs:2.49014152986893,10848);
\draw[draw=none,fill=color0,very thin] (axis cs:2.49014152986893,0) rectangle (axis cs:2.52904999127313,9934);
\draw[draw=none,fill=color0,very thin] (axis cs:2.52904999127313,0) rectangle (axis cs:2.56795845267733,9824);
\draw[draw=none,fill=color0,very thin] (axis cs:2.56795845267733,0) rectangle (axis cs:2.60686691408153,9828);
\draw[draw=none,fill=color0,very thin] (axis cs:2.60686691408153,0) rectangle (axis cs:2.64577537548573,11068);
\draw[draw=none,fill=color0,very thin] (axis cs:2.64577537548574,0) rectangle (axis cs:2.68468383688994,11786);
\draw[draw=none,fill=color0,very thin] (axis cs:2.68468383688994,0) rectangle (axis cs:2.72359229829414,11512);
\draw[draw=none,fill=color0,very thin] (axis cs:2.72359229829414,0) rectangle (axis cs:2.76250075969834,10240);
\draw[draw=none,fill=color0,very thin] (axis cs:2.76250075969834,0) rectangle (axis cs:2.80140922110254,9230);
\draw[draw=none,fill=color0,very thin] (axis cs:2.80140922110254,0) rectangle (axis cs:2.84031768250674,9214);
\draw[draw=none,fill=color0,very thin] (axis cs:2.84031768250674,0) rectangle (axis cs:2.87922614391095,9056);
\draw[draw=none,fill=color0,very thin] (axis cs:2.87922614391095,0) rectangle (axis cs:2.91813460531515,8380);
\draw[draw=none,fill=color0,very thin] (axis cs:2.91813460531515,0) rectangle (axis cs:2.95704306671935,7400);
\draw[draw=none,fill=color0,very thin] (axis cs:2.95704306671935,0) rectangle (axis cs:2.99595152812355,6930);
\draw[draw=none,fill=color0,very thin] (axis cs:2.99595152812355,0) rectangle (axis cs:3.03485998952775,6012);
\draw[draw=none,fill=color0,very thin] (axis cs:3.03485998952775,0) rectangle (axis cs:3.07376845093196,5856);
\draw[draw=none,fill=color0,very thin] (axis cs:3.07376845093196,0) rectangle (axis cs:3.11267691233616,5900);
\draw[draw=none,fill=color0,very thin] (axis cs:3.11267691233616,0) rectangle (axis cs:3.15158537374036,5988);
\draw[draw=none,fill=color0,very thin] (axis cs:3.15158537374036,0) rectangle (axis cs:3.19049383514456,6792);
\draw[draw=none,fill=color0,very thin] (axis cs:3.19049383514456,0) rectangle (axis cs:3.22940229654876,7688);
\draw[draw=none,fill=color0,very thin] (axis cs:3.22940229654876,0) rectangle (axis cs:3.26831075795297,8970);
\draw[draw=none,fill=color0,very thin] (axis cs:3.26831075795297,0) rectangle (axis cs:3.30721921935717,9894);
\draw[draw=none,fill=color0,very thin] (axis cs:3.30721921935717,0) rectangle (axis cs:3.34612768076137,8574);
\draw[draw=none,fill=color0,very thin] (axis cs:3.34612768076137,0) rectangle (axis cs:3.38503614216557,6140);
\draw[draw=none,fill=color0,very thin] (axis cs:3.38503614216557,0) rectangle (axis cs:3.42394460356977,4056);
\draw[draw=none,fill=color0,very thin] (axis cs:3.42394460356977,0) rectangle (axis cs:3.46285306497398,2606);
\draw[draw=none,fill=color0,very thin] (axis cs:3.46285306497398,0) rectangle (axis cs:3.50176152637818,2270);
\draw[draw=none,fill=color0,very thin] (axis cs:3.50176152637818,0) rectangle (axis cs:3.54066998778238,2056);
\draw[draw=none,fill=color0,very thin] (axis cs:3.54066998778238,0) rectangle (axis cs:3.57957844918658,2414);
\draw[draw=none,fill=color0,very thin] (axis cs:3.57957844918658,0) rectangle (axis cs:3.61848691059078,2904);
\draw[draw=none,fill=color0,very thin] (axis cs:3.61848691059078,0) rectangle (axis cs:3.65739537199499,3604);
\draw[draw=none,fill=color0,very thin] (axis cs:3.65739537199499,0) rectangle (axis cs:3.69630383339919,4206);
\draw[draw=none,fill=color0,very thin] (axis cs:3.69630383339919,0) rectangle (axis cs:3.73521229480339,4236);
\draw[draw=none,fill=color0,very thin] (axis cs:3.73521229480339,0) rectangle (axis cs:3.77412075620759,2946);
\draw[draw=none,fill=color0,very thin] (axis cs:3.77412075620759,0) rectangle (axis cs:3.81302921761179,1256);
\draw[draw=none,fill=color0,very thin] (axis cs:3.81302921761179,0) rectangle (axis cs:3.851937679016,356);
\draw[draw=none,fill=color0,very thin] (axis cs:3.851937679016,0) rectangle (axis cs:3.8908461404202,132);
\end{axis}

\end{tikzpicture}
    \caption{Histogram of Euclidean distances $d_\mathcal{X}$ between pairs of inputs on the two coiling spirals dataset.}
    \label{fig:histo}
\end{figure}

In Figure~\ref{fig:spiralerr}, we depict the evolution of the average error rate as a function of the parameter $\varepsilon$. Recall that the extremes for $\varepsilon = 0$ and $\varepsilon = 4$ correspond respectively to Vanilla and \emph{Mixup}.
One can note the significant benefit of \emph{Mixup} and \emph{Local Mixup} over Vanilla. As expected, \emph{Local Mixup} presents a minimum error rate which is significantly smaller than \emph{Mixup}'s error rate. We can note that the minimum is reached with a value of $\varepsilon$ smaller than the first quantile. This means that for this dataset \emph{Mixup} interpolations given above this threshold are either useless or misleading for the network's training.

It is worth pointing out that this toy dataset is particularly suitable to generate contradictory virtual samples. We delve into more complex and real world datasets in the following subsection.

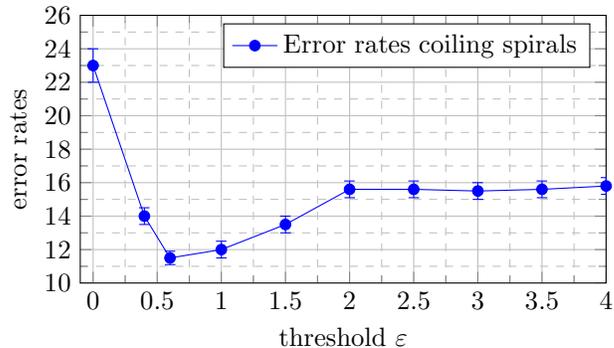
\begin{figure}
\begin{tikzpicture}
\pgfplotsset{%
    width=.45\textwidth,
    height=.3\textwidth
}
\begin{axis}[
    xlabel={threshold $\varepsilon$},
    ylabel={error rates},
    xmin=-0.1, xmax=4,
    ymin=10, ymax=26,
    xtick={0,0.5,1,1.5,2.0,2.5,3.0,3.5,4.0},
    ytick={10,12,14,16,18,20,22,24,26},
    legend pos=north east,
    xmajorgrids=true,
    ymajorgrids=true,
    xminorgrids=true,
    yminorgrids=true,
    minor tick num=1,
    width=.5\textwidth,
    minor grid style=dashed,
]

\addplot[
    color=blue,
    mark=*,
    blue,
    error bars/.cd, 
    y dir=both, 
    y explicit
    ]
    coordinates { (0,23)+=(0,1) -= (0,1)
    (0.4,14)+=(0.4,0.5) -= (0.4,0.5)
    (0.6,11.5)+=(0.6,0.4) -= (0.6,0.4)
    (1,12)+=(1,0.5) -= (1,0.5)
    (1.5,13.5)+=(1.5,0.5) -= (1.5,0.5)
    (2,15.6)+=(2,0.5) -= (2,0.5)
    (2.5,15.6)+=(2.5,0.5) -= (2,0.5)
    (3,15.5)+=(3,0.5) -= (2,0.5)
    (3.5,15.6)+=(3.5,0.5) -= (3.5,0.5)
    (4,15.8)+=(4,0.5) -= (4,0.5)
    };
    \legend{Error rates coiling spirals }
    
\end{axis}
\end{tikzpicture}
    \caption{Error rate as a function of $\varepsilon$ for the two coiling spirals dataset. Values are averaged over 1000 runs. Extremes correspond respectively to Vanilla ($\varepsilon= 0$) and \emph{Mixup} ($\varepsilon > 4$).}
    \label{fig:spiralerr}
\end{figure}

\subsection{High dimension}
\subsubsection{Lipschitz lower bound}
To illustrate the impact of $\varepsilon$ on the optimal Lipschitz constant, we use the dataset CIFAR-10~\citep{Cifar10} which is made of small images of size 32x32 pixels and 3 colors. There are 50,000 images in the training set corresponding to 10 classes.

We are interested in showcasing the evolution of $Q(D)$ when varying $\varepsilon$. The results are depicted in Figure~\ref{fig:evolutionofds}.

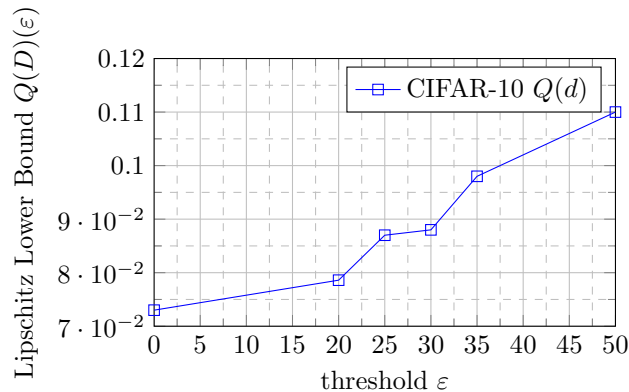
\begin{figure}[h!]
    \centering
\begin{tikzpicture}
\pgfplotsset{%
    width=.45\textwidth,
    height=.3\textwidth
}
\begin{axis}[
    xlabel={threshold $\varepsilon$},
    ylabel={Lipschitz Lower Bound $Q(D)(\varepsilon)$},
    xmin=0, xmax=50,
    ymin=0.07, ymax=0.12,
    xtick={0,5,10,15,20,25,30,35,40,45,50},
    ytick={0.07,0.08,0.09,0.1,0.11,0.12},
    legend pos=north east,
    xmajorgrids=true,
    ymajorgrids=true,
    xminorgrids=true,
    yminorgrids=true,
    minor tick num=1,
    minor grid style=dashed,
]

\addplot[
    color=blue,
    mark=square,
    ]
    coordinates
    {(0, 0.073)(20,0.0786)(25,0.087)(30,0.0880)(35,0.0980)(50,0.11)
    };
    \legend{CIFAR-10 $Q(d)$ }
    
\end{axis}
\end{tikzpicture}
    \caption{Evolution of $Q(D)$ on the dataset CIFAR10. Note that $\varepsilon = 0$ corresponds to Vanilla. $\varepsilon = 50$ corresponds to classical \emph{Mixup}.}
     \label{fig:evolutionofds}
\end{figure}
For classical \emph{Mixup} we obtained $Q(D) = 0.11$ and for Vanilla $Q(D) = 0.073$. Note that these two extremes are reached with \emph{Local Mixup} when $\varepsilon = 0$ and  $\varepsilon \geq 50$.

We observe that $\varepsilon$ can be used to smoothly tune the lower bound $Q(D)$. In practice, a lower $Q(D)$ is preferable, but this only accounts for the optimal Lipschitz constant. Larger values of $\varepsilon$ lead to larger training sets and thus potentially better generalization.

\subsubsection{Experiments on classification dataset}

We now test our proposition on different classification datasets and architectures.  We consider the datasets CIFAR10~\citep{Cifar10}, Fashion-MNIST~\citep{Fashion} and SVHN~\citep{svhn}. Fashion-MNIST is composed of clothes images of size 28x28 pixels (grayscale) . There are 60,000 images in the training set corresponding to 10 classes. SVHN is a real-world image dataset made of small cropped digits of size 32x32 pixels and 3 colors. There are 73257 digits in the training set corresponding to 10 classes. For these tests, we use a smooth decreasing exponential graph parametrized by $\alpha$.

For CIFAR10, we implement a ResNet18~\citep{resnet} as in~\citep{zhang2017mixup}, and average the error rates over 100 runs. We report the mean and confidence interval at 95\%. We observed that \emph{Local Mixup} with a value of $\alpha = 0.003$ showed a smaller error rate than the Vanilla network and \emph{Mixup}, with disjoint confidence intervals.
For Fashion MNIST, we implement a Densenet~\citep{densenet} and average the error rates over 10 runs. We also report the mean and confidence intervals at 95\%. Again, \emph{Local Mixup} with a value of $\alpha = 1e-3$ presents a smaller error rate than both the baseline and \emph{Mixup}. Note that for this dataset and this network architecture \emph{Mixup} impacts negatively the error rate, suggesting that on this dataset \emph{Mixup} creates spurious interpolations as discussed in~\cite{guo2019mixup}.
For SVHN we implement a LeNet-5\citep{lenet} architecture (3 convolution layers). Again, \emph{Local Mixup} performs better than both Vanilla and \emph{Mixup}.
\begin{table}[]
    \centering
        \caption{Error rates (\%) on CIFAR10, Fashion-MNIST and SVHN. Values are averaged on 100 runs for Cifar10 and 10 runs for Fashion-MNIST and SVHN. Mean errors with their confidence interval are given.}
    \begin{tabular}{c c c}
     \hline
    \textbf{MODEL} & \textbf{CIFAR-10} & \textbf{ERROR \%} \\
    \hline
&Baseline & $4.98 \pm  0.03$  \\
Resnet18 &Mixup & $4.13   \pm 0.03$   \\
&LM($\alpha = 3e-3$) & \textbf{ 4.03 $\pm$ 0.03}\\
\hline 
&\textbf{FASHION-MNIST}&  \\
\hline 
&Baseline & $6.20 \pm 0.2$  \\
DenseNet&Mixup & $6.36   \pm 0.16$   \\
&LM ($\alpha = 1e-3$) & \textbf{5.97 $\pm$ 0.2} \\
\hline
&\textbf{SVHN}&  \\ 
\hline
&Baseline & $10.01 \pm 0.15$  \\
LeNet &Mixup & $8.31  \pm 0.14$   \\
&LM ($\alpha = 5e-2$) & \textbf{8.20 $\pm$ 0.13} \\
    \end{tabular}
    \label{tab:ErrCifar10}
\vspace{-0.1cm}
\end{table}

For these experiments, we also tried to use a $K$-nearest neighbor graph or a thresholded graph but without being able to achieve smaller error rates compared to \emph{Mixup} or even Vanilla. This may indicate that some segments generated by \emph{Mixup} are important to act as a regularizer during training even if some of them may generate manifold intrusions. By tuning $\alpha$, we weigh the importance of this regularization.

\subsubsection{Discussion}
Experiments in both low and high dimensions demonstrated the capacity of \emph{Local Mixup} to outperform \emph{Mixup} thanks to the use of locality. 
Still, the choice of the added hyper-parameter ($\alpha$, $\varepsilon$ or $K$) is essential and data dependent. For now, we reported results selecting the parameter leading to the best test error rate among a small number of possibilities. In future work we would like to rely on quantitative information given on the topology such as the histogram of the distance or persistence diagrams\citep{topologicaldatanalysis} to tune these hyper-parameters.

Note also that to embed the notion of locality we decided to use the Euclidean metric, although in general datasets lie in nonlinear manifolds. On CIFAR10 for example, in~\citep{KnearestMixup} the authors show that it is possible to achieve classification scores significantly better than the chance level using the Euclidean metric, but very far from state-of-the-art. There would be many possibilities to improve over using the Euclidean metric, including using pullback metrics~\citep{pullback,pullbacklatent} given by the euclidean distance between the samples once in the feature space corresponding to the penultimate layer.

\section{Conclusion}

In this paper, we introduced a methodology called \emph{Local Mixup}, in which pairs of samples are interpolated and weighted in the loss depending on the distance between them in the input domain.
This methodology comes with a hyper-parameter that allows to provide a continuous range of solutions between Vanilla and classical \emph{Mixup}.
Using a simple framework, we showed that \emph{Local Mixup} can control the bias/variance trade-off of trained models. In more general settings, we showed that \emph{Local Mixup} can tune a lower bound on the Lipschitz constant of the trained model.
We used real world datasets to prove the ability of \emph{Local Mixup} to achieve better generalization, as measured using the test error rate, than Vanilla and classical \emph{Mixup}.

Overall, our methodology introduces a simple way to incorporate locality notions into \emph{Mixup}. We believe that such a notion of locality is beneficial and could be leveraged to a greater level in future work, or could be incorporated to the various \emph{Mixup} extensions that have been proposed in the community.
In future work, we would like to investigate further the choice of the graph, the choice of the hyper-parameter that comes with it, and trainable versions of \emph{Local Mixup}. Extending the theoretical results to more general contexts would definitely allow to gain further intuition on the effect of locality on \emph{Mixup}.

\bibliography{biblio}
\end{document}